\documentclass[11pt]{article}

\usepackage{amsmath}
\usepackage{mathtools}
\usepackage{amsfonts}
\usepackage{amssymb}
\usepackage{graphicx}
\usepackage{hyperref}
\usepackage{amsthm}

\newtheorem{theorem}{Theorem}
\newtheorem{lemma}[theorem]{Lemma}

\newtheorem{proposition}[theorem]{Proposition}

\newtheorem{corollary}[theorem]{Corollary}

\newtheorem*{problem}{Problem}

\def\final{0}  

\ifnum\final=0  
\newcommand{\lnote}[1]{[{\tiny Luis: \bf #1}]\marginpar{*}}
\newcommand{\nnote}[1]{[{\tiny navin: \bf #1}]\marginpar{*}}
\newcommand{\sidecomment}[1]{\marginpar{\tiny #1}}
\else 
\newcommand{\lnote}[1]{}
\newcommand{\nnote}[1]{}
\newcommand{\sidecomment}[1]{}
\fi  

\newcommand{\RR}{\mathbb{R}}
\newcommand{\eps}{\epsilon}
\newcommand{\family}{\mathcal{P}}
\newcommand{\familyc}[1]{\mathcal{P}_{#1}}
\newcommand{\familycc}[2]{\mathcal{P}_{#1}^{#2}}
\newcommand{\suchthat}{\mathrel{:}}
\newcommand{\ones}{\mathbf{1}}
\newcommand{\card}[1]{\lvert#1\rvert}

\newcommand{\lrabs}[1]{\left\lvert{#1}\right\rvert}
\newcommand{\poly}{\mathrm{poly}}
\newcommand{\vol}[1]{\lvert#1\rvert}
\newcommand{\dist}{\mathrm{dist}}

\begin{document}
\title{Learning convex bodies is hard}

\author{Navin Goyal\\ Georgia Tech\\ navin001@gmail.com \and Luis Rademacher\\ Georgia Tech\\ lrademac@cc.gatech.edu}
\date{}
\maketitle
\begin{abstract}
We show that learning a convex body in $\RR^d$, given random samples
from the body, requires $2^{\Omega(\sqrt{d/\eps})}$ samples. By learning a convex body we mean
finding a set having at most $\eps$ relative symmetric difference with the input body.
To prove the lower bound we construct a hard to learn family of convex bodies.  
Our construction of this family is very simple and based on error correcting codes.  
\end{abstract}

\section{Introduction}
We consider the following problem: Given uniformly random points 
from a convex body in $\RR^d$, we would like to approximately learn the body with
as few samples as possible.  In this question, and throughout this paper, we are 
interested in the number of samples but not in the computational requirements for 
constructing such an approximation.  
Our main result will show that this needs about $2^{\Omega(\sqrt{d})}$ samples.
This problem is a special case of the statistical problem of inferring 
information about a probability distribution from samples. For example,
one can approximate the centroid of the body with a sample of size roughly
linear in $d$. On the other hand, a sample of size polynomial in $d$ is not 
enough to approximate the volume of
a convex body within a constant factor (\cite{Eldan09}, 
and see Section \ref{sec:discussion} here for a discussion). 
Note that known approximation algorithms
for the volume (e.g., \cite{dfk}) do not work in this setting as they need
a membership oracle and random points from various carefully chosen subsets of the 
input body.
 
Our problem also relates to work in learning theory (e.g., \cite{Vempala04, KlivansOS08}), where one is 
given samples generated according to (say) the Gaussian distribution and each sample is labeled 
``positive'' or ``negative'' depending on whether it belongs
to the body. Aside from different distributions, another difference between the learning setting of 
\cite{KlivansOS08} and ours is that in ours one gets only positive examples.  
Klivans~et~al.~\cite{KlivansOS08}
give an algorithm and a nearly matching lower bound for learning convex bodies 
with labeled samples chosen according to the Gaussian distribution.  Their algorithm takes time 
$2^{\tilde{O}(\sqrt{d})}$ and they also show a lower bound of $2^{\Omega(\sqrt{d})}$.

The problem of learning convex sets from uniformly random samples from them was raised by
Frieze~et~al.~\cite{FriezeJK96}.  They gave a polynomial time algorithm for learning parallelopipeds.  
Another somewhat related direction is the work on the learnability of discrete distributions by
Kearns~et~al.~\cite{Kearnsetal94}.

Our lower bound result (like that of \cite{KlivansOS08}) also allows for membership oracle queries. 
Note that it is known that estimating the volume of convex bodies requires an exponential number of 
membership 
queries if the algorithm is \emph{deterministic} \cite{BaranyF87}, which implies that learning bodies
requires an exponential number of membership queries because if an algorithm can learn the body then it
can also estimate its volume. 

To formally define the notion of learning we need to specify a distance 
$d(\cdot, \cdot)$ between bodies.  A natural choice 
in our setting is to consider the total variation distance of the uniform distribution on each body
(see Section \ref{sec:preliminaries}).

We will use the term \emph{random oracle of a convex body $K$} for a black box that 
when queried outputs a uniformly random point from $K$.

\begin{theorem} \label{thm:main}
There exists a distribution $\mathcal{D}$ on the set of convex bodies in $\RR^d$
satisfying the following: Let $\mathrm{ALG}$ be a randomized algorithm that,
given a random convex body $K$ according to $\mathcal{D}$, makes
at most $q$ total queries to random and membership oracles of $K$ and outputs a
set $C$ such that, for $8/d \leq \eps \leq 1/8$,
\[
\Pr (d(C, K) \leq \eps) \geq 1/2
\]
where the probability is over $K$, the random sample and any randomization by $ALG$. Then
\[
    q \geq 2^{\Omega(\sqrt{d/\eps}) }.
\]
\end{theorem}

Remarkably, the lower bound of Klivans~et~al.~\cite{KlivansOS08} is numerically essentially identical to 
ours ($2^{\Omega(\sqrt{d})}$ for $\eps$ constant).  Constructions similar to theirs are possible for our
particular scenario~\cite{Kalai06}.  We believe that our argument is considerably simpler, and
elementary compared to that of \cite{KlivansOS08}.  Furthermore, our construction of the hard to learn family 
is explicit.  Our construction makes use of error correcting codes.  To our knowledge, this connection with 
error correcting codes is new in such contexts and may find further
applications.  See Section~\ref{sec:discussion} for some further comparison.

\paragraph{An informal outline of the proof.} The idea of the proof is to find a large family of convex
bodies in $\RR^n$ satisfying two conflicting goals:
(1) Any two bodies in the family are almost disjoint; (2) and yet they look alike in the sense
that a small sample of random points from any such body
is insufficient for determining which one it is. Since any two bodies
are almost disjoint, even approximating a body would allow one to
determine it exactly. This will imply that it is also hard
to approximate.

We first construct a family of bodies that although not almost disjoint, have sufficiently large
symmetric difference.  We will then be able to construct a family with almost disjoint bodies by
taking products of bodies in the first family.

The first family is quite natural (it is described formally in Sec.~\ref{sec:inner}).
Consider the cross polytope $O_n$ in $\RR^n$ (generalization of the octahedron to $n$ dimensions:
convex hull of the vectors $\{\pm e_i: i \in [n]\}$, where $e_i$ is the unit vector in $\RR^n$ with the 
$i$th coordinate $1$ and the rest $0$).  A \emph{peak} attached to a facet $F$ of $O_n$
is a pyramid that has $F$ as its base and has its other vertex outside $O_n$ on the normal to $F$ going 
through its centroid.
If the height of the peak is sufficiently small then attaching peaks to any subset of the $2^n$
facets will result in a convex polytope.  We will show later that we can choose the height so that
the volume of all the $2^n$ peaks is $\Omega(1/n)$ fraction of the volume of $O_n$.
We call this family of bodies $\family$.  [We remark that our construction of cross-polytopes with peaks
has resemblance to a construction in \cite{RademacherV04} with different parameters, but there does not 
seem to be any connection between the problem studied there and the problem we are interested in.]

Intuitively, a random point in a body from this family tells one that if the
point is in one of the peaks then that peak is present,
otherwise one learns nothing.  Therefore if the number of queries
is at most a polynomial in $n$, then one learns nothing about most of
the peaks and so the algorithm cannot tell which body it got.

But these bodies do not have large symmetric difference
(can be as small as a $O(1/(n2^n))$ fraction of the cross polytope if the two bodies differ in just
one peak)
but we can pick a subfamily of them having pairwise symmetric difference
at least $\Omega(1/n)$ by picking a large random subfamily.  We will do it slightly differently
which will be more convenient for the proof: Bodies in $\family$ have one-to-one correspondence with
binary strings of length $2^n$: each facet corresponds to a coordinate of the string which takes value
$1$ if that facet has a peak attached, else it has value $0$.  To ensure that any two bodies in 
our family differ in many peaks it suffices to ensure that their corresponding strings have large
Hamming distance.  Large sets of such strings are of course furnished by good error correcting codes.

From this family we can obtain another family of almost disjoint bodies by
taking products, while preserving the property that polynomially many random samples do not tell the
bodies apart.  This product trick (also known as tensoring) has been used many times before,
in particular for amplifying hardness, but we are not aware of its use in a setting similar to ours. 
Our construction of the product family also resembles the operation of concatenation in coding theory. 

\textbf{Acknowledgments.} We are grateful to Adam Kalai and Santosh Vempala
for useful discussions.

\section{Preliminaries}\label{sec:preliminaries}
Let $K,L \subseteq \RR^n$ be bounded and measurable. We 
define a distance $\dist(K,L)$ as the total variation distance between the
uniform distributions in $K$ and $L$, that is,
\[
    \dist(K,L) =
    \begin{cases}
        \frac{\vol{K \setminus L}}{\vol K} & \text{if $\vol K \geq \vol L$} \\
        \frac{\vol{L \setminus K}}{\vol L} & \text{if $\vol L > \vol K$.}
    \end{cases}
\]

We will use $\vol{A}$ to denote the volume of sets $A \subset \RR^n$, and also to denote the cardinality 
of finite sets $A$; which one is meant in a particular case will be clear from the context.  

Let $\ones$ denote the vector $(1, \dotsc, 1)$.  ``$\log$'' denotes logarithm with base $2$. 

We will need some basic definitions and facts from coding theory; see, e.g., \cite{vanLint}.
For a finite alphabet $\Sigma$, and word length $n$, a \emph{code} $C$ is a subset of $\Sigma^n$.
For any two codewords $x, y \in C$, distance $\dist(x,y)$ between them is defined by
$\dist(x,y):= |\{i \in [n] :  x_i \neq y_i\}|$. 
The \emph{relative minimum distance} for code $C$ is
$\min_{x, y \in C, x \neq y}\dist(x,y)/n$.  For $\Sigma = \{0,1\}$, the \emph{weight} of a codeword $x$ is
$|\{i \in [n]\;:\; x_i \neq 0\}|$.
Define $V_q(n,r) := \sum_{i=0}^{r} \binom{n}{i}(q-1)^i$.  The following is well-known and easy to prove:
\begin{theorem}[Gilbert--Varshamov] \label{thm:GV}
For alphabet size $q$, code length $n$, and minimum distance $d$, there exists a code of size at least
$q^n/V_q(n,d-1)$.
\end{theorem}
When the alphabet is $\Sigma = \{0,1\}$, we define the \emph{complement} $\bar c$
of a codeword $c \in C$ as $\bar c_i := 1-c_i$.

\section{A hard to learn family of convex bodies}

The construction proceeds in two steps.  In the first step we construct a large subfamily of
$\family$ such that the
relative pairwise symmetric difference between the bodies is $\Omega(1/n)$.  This symmetric
difference is however not sufficiently large for our lower bound.  The second step of the
construction amplifies the symmetric difference by considering products of the bodies from the
first family.

\subsection{The inner family: Cross-polytope with peaks} \label{sec:inner}

We first construct a family with slightly weaker properties.
The family consists of what we call ``cross polytope with peaks''.
The $n$-dimensional cross polytope $O_n$ is the convex hull of the
$2n$ points $\{ \pm e_i, i \in [n]\}$. Let $F$ be a facet of $O_n$, and
let $c_F$ be the center of $F$. The \emph{peak} associated to $F$
is the convex hull of $F$ and the point $\alpha c_F$, where $\alpha > 1$
is a positive scalar defined as follows: $\alpha$ is picked as large as possible
so that the union of the cross polytope and all $2^n$ peaks is a convex
body. A \emph{cross polytope with peaks} will then be the union of the cross
polytope and any subfamily of the $2^n$ possible peaks. The set of all
$2^{2^n}$ bodies of this type will be denoted $\family$.
By fixing of an
ordering of the facets of the cross polytope, there is a one-to-one
correspondence between the cross polytope with peaks and 0--1 vectors with
$2^n$ coordinates.

Let $P$ denote the cross polytope with all $2^n$ peaks.
We will initially choose $\alpha$ as large as possible so that the
following condition---necessary for convexity of $P$ but not clearly
sufficient---is satisfied: for every pair of adjacent facets $F$, $G$ of $O_n$,
the vertex of each peak is in the following halfspace:
the halfspace containing $O_n$ and whose boundary is the hyperplane
orthogonal to the (vector connecting the origin to the) center of $F \cap G$, and containing $F \cap G$.
A straightforward computation shows that $\alpha = n/(n-1)$ for this condition.
This implies by another easy computation that the volume of all the peaks is
$\vol{O_n}/(n-1)$. 
We will now show that this weaker condition on $\alpha$ is actually sufficient
for the convexity of $P$ and any cross polytope with peaks.

\begin{proposition}
Every set in $\family$ is convex.
\end{proposition}
\begin{proof}
Let $Q$ be
the intersection of all halfspaces of the form
\[
    \{ x \in \RR^n \suchthat a \cdot x \leq 1 \}
\]
where $a \in \RR^n$ is a vector having entries in $\{-1,0,1\}$ and exactly one
zero entry. Equivalently, the boundary of each such halfspace is a hyperplane
orthogonal to the center of some $(n-2)$-dimensional face of $O_n$ and containing
that face.

In the rest of the proof we will show that $P = Q$, which gives the convexity
of $P$. This equality implies
that a cross polytope with only some peaks is also convex: any such body
can be obtained from $P$ by intersecting $P$ with the hafspaces induced by the
facets of $O_n$ associated to the missing peaks, and it is easy to see from the
definition of the peaks that each such intersection removes exactly one peak.

It is clear that $P \subseteq Q$.
For the other inclusion, let $x \in Q$.
By symmetry we can assume $x \geq 0$. If $\sum x_i \leq 1$, then
$x \in O_n \subseteq P$. If $\sum x_i > 1$, we will show that $x$ is in the peak
of the positive orthant. We would like to write $x$ as a convex combination of
$e_1, \dotsc, e_n$ and the extra vertex of the peak, $v = \ones /(n-1)$.
Let $\mu = (n-1) ((\sum x_i) -1) > 0$. We want a vector
$\lambda = (\lambda_1, \dotsc, \lambda_n)$ such that $x$ is a convex
combination of the vertices of the peak:
\[
x = \mu v + \sum \lambda_i e_i  = \mu v + \lambda,
\]
that is, $\lambda = x + \ones - \ones \sum x_i$. It satisfies
\[
\mu + \sum \lambda_i = 1
\]
and $\lambda_i = x_i + 1 - \sum x_i$, and this is non-negative:
By definition of $Q$ we have for all $j \in [n]$
\[
\sum x_i \leq 1 + x_j.
\]
This shows that $x$ belongs to the peak in the positive orthant.
\end{proof}

For notational convenience we let $N := 2^n$.
Recall that we identify bodies in $\family$ with binary strings in $\{0,1\}^N$.
Let $C \subset \{0, 1\}^{N}$ be a code with relative
minimum distance at least $1/4$.
To simplify computations involving distance ``$\dist$'' between bodies,
it will be convenient to have the property that all codewords in $C$
have weight $N/2$.  We can ensure this easily as follows.
Let $\tilde{C} \subset \{0, 1\}^{N/2}$
be a code with relative minimum distance at least $1/4$,
then set $C := \{(c, \bar{c}): c \in \tilde{C}\}$.
Clearly $|C| = |\tilde{C}|$.
By Theorem~\ref{thm:GV} we can choose $\tilde{C}$ such that $|\tilde{C}| \geq 2^{c_1 N}$,
for a positive constant $c_1$.  We fix $C$ to be such a code, i.e. a code with
relative minimum distance at least $1/4$, size $2^{c_1N}$, and all codewords with weight $N/2$.

We define the family $\familyc{C}$ as the family consisting of  bodies in $\family$ corresponding to
codewords in $C$. As all codewords in $C$ have the same weight,
we have that all bodies in $\familyc{C}$ have the same volume.
Recall that the volume of each peak is $\frac{\vol{O_n}}{2^n(n-1)}$.  Therefore
for distinct $P, Q \in \familyc{C}$ the volume of the symmetric difference of
$P$ and $Q$ is at least $\frac{\vol{O_n}}{4(n-1)}$.  

\subsection{The outer family: The product construction}

Let $C'$ be a code with codewords of length $k$ and minimum distance at least $k/2$ on the alphabet
$\familyc{C}$.  That is, codewords in $C'$
can be represented as $(B_1, \ldots, B_k)$, where $B_i \in \familyc{C}$ for $i=1, \ldots, k$.
The product family $\familycc{C}{C'}$ corresponding to code $C'$, has $|C'|$ bodies in $\RR^{kn}$,
one for each codeword.  The body for codeword $(B_1, \ldots, B_k) \in C'$ is simply
$B_1 \times \ldots \times B_k$.

Clearly $|\familycc{C}{C'}| = |C'|$.  Using Theorem~\ref{thm:GV} we can choose $C'$ such that
$|C'| \geq q^k/V_{q}(k,k/2)$.  Now note that
\begin{align*}
V_q(k, k/2) &= \sum_{i=0}^{k/2} \binom{k}{i}(q-1)^i \leq (q-1)^{k/2} \sum_{i=0}^{k/2} \binom{k}{i} \\
&< 2^{k} (q-1)^{k/2} < (4q)^{k/2}.
\end{align*}

Therefore $q^k/V_q(k, k/2) > q^k/(4q)^{k/2} = (q/4)^{k/2}$.
Setting $q = 2^{c_1 N}$, we get $|C'| > 2^{(c_1N -2)k/2} > 2^{c_2kN}$, for constant $c_2>0$, assuming $N$ is
sufficiently large. We just showed:
\begin{lemma}\label{lem:cardinality}
$\card{\familycc{C}{C'}} > 2^{c_2 k 2^n}$.
\end{lemma}

The following lemma shows that the bodies in $\familycc{C}{C'}$ are almost pairwise disjoint.
\begin{lemma} \label{lem:distance}
For distinct $A, B \in \familycc{C}{C'}$ we have
\[
\dist(A,B) = 1-\frac{|A \cap B|}{|A|} > 1 - e^{-k/(16 n)}.
\]
\end{lemma}
\begin{proof}

We constructed $\familycc{C}{C'}$ so that all bodies in it have the same volume. 
This implies
\[
    \dist(A,B) = 1-\frac{|A \cap B|}{|A|}.
\]
Let $A = A_1 \times \ldots \times A_k$ and $B = B_1 \times \ldots \times B_k$.  Then
\[
\frac{|A \cap B|}{|A|} = \frac{|A_1 \cap B_1| \times \ldots \times |A_k \cap B_k|}{|A_1|\times \ldots
\times |A_k|}.
\]

Since the minimum relative distance in $C$ is at least $1/4$
and the weight of each codeword is $N/2$, we have that for $A_i \neq B_i$ the number of peaks
in $A_i \cap B_i$ is at most $2^n \cdot 3/8$. Hence 
\[
\frac{|A_i\cap B_i|}{|A_i|} \leq \frac{1+3/(8(n-1))}{1+1/(2(n-1))}.
\]
Since the minimum distance of $C'$ is at least $k/2$, we have $A_i \neq B_i$ for at least $k/2$ values of $i$
in $[k]$.  Therefore we get
\begin{align*}
\frac{|A \cap B|}{|A|} &\leq \left(\frac{1+3/(8(n-1))}{1+1/(2(n-1))}\right)^{k/2} \\
&\leq \left(1-\frac{1}{8n}\right)^{k/2} < e^{-k/(16 n)} .
\end{align*}
\end{proof}

\section{Proof of the lower bound}

\begin{proof}[Proof of Theorem \ref{thm:main}.]
We will make use of the family $\familycc{C}{C'}$ that we constructed in the previous section. 
Recall that the bodies in this family live in $\RR^d$, for $d := kn$.  For this proof we will think
of $d$ as fixed and we will choose $n$ appropriately for the lower bound proof.
By a straightforward but tedious argument it is enough
to prove the theorem assuming that $d$ is a power of 2. 

We will use Yao's principle (see, e.g., \cite{MRaghavan}). To this end, we will first show that
the interaction between an algorithm and the oracles can be
assumed to be ``discrete'', which in turn will imply that effectively
there is only a finite number of deterministic algorithms that make at
most $q$ queries. The discretization of the oracles also serves a
second purpose: that we can see deterministic algorithms as finite
decision trees and use counting arguments to show a lower bound on the 
query complexity.

Fix a body $K$ from $\familycc{C}{C'}$.
Suppose that a randomized algorithm has access to the following
discretizations of the oracles:
\begin{itemize}
\item A discrete random oracle that generates a random point $X = (X_1, \dotsc,
X_k)$ from $K = \prod_i K_i$ and, for each $i \in [k]$
outputs whether $X_i$ lies in
the corresponding cross-polytope or in which peak it lies.
\item A discrete membership oracle that when given a
sequence of indices of peaks $I = (i_1, \dotsc, i_k)$ outputs, for
each $i \in [k]$, whether peak $i$ is present in $K_i$.
\end{itemize}

\emph{Claim:}
A randomized algorithm with access to discrete versions of the oracles can
simulate a randomized algorithm with access to continuous oracles with the
same number of queries.

\emph{Proof of claim:}
We will show it for bodies in $\family$, i.e. cross polytopes with peaks;
the generalization of this argument to product bodies
in $\familycc{C}{C'}$ is straightforward.
Let $A$ and $B$ be the algorithm with access to the continuous and discrete
oracles respectively. Algorithm $B$ acts as $A$, except when $A$ invokes an oracle,
where $B$ will do as follows: When $A$ makes a query $p$ to
the continuous membership oracle, $B$ will query the peak that contains $p$
(we can assume that $p$ lies in a peak, as otherwise the query provided no
new information).  Now suppose that $A$ makes a query to the continuous random
oracle and gets a point $p$.  Then $B$ makes a query to the discrete random
oracle.
$B$ then generates a uniformly random point $p'$ in the region that it got
from the oracle.  Clearly $p'$ has the same distribution as $p$, namely
uniform distribution on the body.  

If we see deterministic algorithms as decision trees,
it is clear that there are only a finite number of deterministic algorithms
that make at most $q$ queries to the discrete oracles of $K$.
Thus, by Yao's principle, for any distribution $\mathcal{D}$ on inputs,
the probability of error of any randomized algorithm
against $\mathcal{D}$ is at least the probability of error
of the best deterministic algorithm against $\mathcal{D}$.

Our hard input distribution $\mathcal{D}$ is the uniform distribution over $\familycc{C}{C'}$. 
Now, in the decision tree associated to a deterministic algorithm,
each node associated to a membership query has two children 
(either the query point is in the body or not), 
while a node associated to a random sample has at most 
$(2^n+1)^k$ children (the random sample can lie in one of the 
$2^n$ peaks or in $O_n$, for each factor in the product body).
Thus, if the algorithm makes at most $q$ queries in total, then the
decision tree has at most $(2^n+1)^{kq}$ leaves. These leaves induce
a partition of the family of inputs $\familycc{C}{C'}$.  By Lemma~\ref{lem:distance},
the distance between any pair of bodies is at least $1 - e^{-k/16 n} = 1 - e^{-d/16 n^2}$, 
where $n$ is chosen so that
the output of the algorithm can be within $\eps$ of at most
one body in each part of the partition. That is,
\[
2 \eps < 1-e^{-k/16 n},
\]
which implies that we should take
\[
    n < 4 \sqrt{\frac{d}{\ln{\frac{1}{1-2\eps}}}}.
\]
As $d=kn$ is a power of 2, we can satisfy the previous inequality
and the integrality constraints of $k$ and $n$ by using our assumption
that $8/d \leq \eps \leq 1/8$ and letting $n$ be a power of 2 such that
\[
  2 \leq \sqrt{\frac{d}{\ln{\frac{1}{1-2\eps}}}} \leq n 
  < 4  \sqrt{\frac{d}{\ln{\frac{1}{1-2\eps}}}} \leq d.
\]
By Lemma~\ref{lem:cardinality}, the total number of bodies is
\[
\card{\familycc{C}{C'}} \geq 2^{c_2 k 2^n}.
\]
This implies that the probability of error is at least
\[
    1-\frac{(2^n+1)^{kq}}{\card{\familycc{C}{C'}}} 
    \geq 1-\frac{(2^n+1)^{kq}}{2^{c_2 k 2^n}}.
\]
If we want this error to be less than a given $\delta$, then
for some $c_3, c_4 > 0$ we need
\begin{align*}
q &\geq c_3 \left(\frac{\log(1-\delta)}{kn} + \frac{2^n}{n}  \right)\\
    &\geq c_3 \left( \frac{\log(1-\delta)}{d} + \frac{1}{\sqrt{d}}\sqrt{\log{\frac{1}{1-2\eps}}} \cdot 2^{\sqrt{\frac{d}{\log{\frac{1}{1-2\eps}}}}}\right) \\
    &\geq c_3 \left( \frac{\log(1-\delta)}{d} + \sqrt{\frac{\eps}{d}} \cdot 2^{\sqrt{\frac{d}{\log{\frac{1}{1-2\eps}}}}}\right).
\end{align*}
For $\delta =1/2$ and $\eps\leq 1/4$ this implies
\[
    q \geq 2^{\Omega(\sqrt{d/\eps})}.
\]
\end{proof}

\section{Discussion}\label{sec:discussion}

Informally, our construction of $\familycc{C}{C'}$ can be thought of as ``codes'' in $\RR^d$, namely sets in 
$\RR^d$ that are far from each other; the difficulty in the construction of such codes
comes from the requirements of convexity and that the distributions of polynomially many random samples
look alike.  
By using slightly more involved arguments we can handle $\eps$ arbitrarily close to $1$ and prove a 
similar lower bound.  It is not clear if such a lower bound is possible for other learning settings that
have been studied in the past, e.g. labeled samples from Gaussian distribution.  
Unlike that setting, we do not know a matching upper bound for learning convex
bodies in our model.  

Our construction of the hard family is more ``explicit'' than that of \cite{KlivansOS08}:  The hard family 
they construct is obtained by a probabilistic argument; our construction can be made explicit by using good
error correcting codes.  

\medskip

We mention here a somewhat surprising corollary of our result, without detailed proof or precise numerical constants.
Informally, it shows instability of the reconstruction of a convex body as a function of the volume
of its intersection with halfspaces, relative to its volume. It is an exercise to see that knowledge of 
$|K \cap H|/|K|$ for every halfspace $H$ uniquely determines $K$.
Moreover, given $d^c$ (for some fixed constant $c > 0$) random 
samples from a convex body $K \subseteq \RR^d$, with high probability we can estimate $\vol{K \cap H}/|K|$ for 
\emph{all} half-spaces $H$ within additive error of $O(1/d^{c'})$, where $c'$ is a positive constant depending on $c$.  
This can be proved using standard arguments about
$\epsilon$-approximations and the fact that the VC-dimension of halfspaces in $\RR^d$ is $d+1$.  

We say that two convex bodies $K$ and  $L$ are $\alpha$-\emph{halfspace-far} if there is a halfspace $H$ such that 
$||K \cap H|/|K| - |L \cap H|/|L||> \alpha$.  Thus if we choose some $t < c'$ and $K$ and $L$ are $1/d^t$-halfspace-far, 
then we can detect this 
using $d^c$ random points, with high probability.  Now, we claim that there is a pair of bodies in $\familycc{C}{C'}$ 
that is not far.  For otherwise, all pairs would be far and we would be able to distinguish every body in $\familycc{C}{C'}$ 
from every other body in $\familycc{C}{C'}$ with a sample of size $d^c$, 
and thus learn it.  But as we have proved, this is impossible.  So we can conclude 
that there are two bodies in $\familycc{C}{C'}$ that are not $1/d^t$-halfspace-far, i.e. they are $1/d^t$-halfspace-close.  
This gives:
\begin{corollary}
For any constant $t>0$ and sufficiently large $d$ there exist two convex bodies 
$K,L \subseteq \RR^d$ such that $K,L$ are $1/d^t$-halfspace-close: for every halfspace $H$
\[
\lrabs{\frac{\vol{K \cap H}}{\vol{K}} - \frac{\vol{L \cap H}}{\vol{L}}} \leq \frac{1}{d^t},
\]
but $\dist(K, L) > 1/8$.
\end{corollary}


An earlier version of this manuscript mentioned the problem of whether the volume of a convex body in $\mathbb{R}^d$
can be estimated from $\poly(d)$ uniformly random samples. 
Very recently, Ronen Eldan~\cite{Eldan09} has answered this in the negative. 
His result provides a probabilistic construction of a family of convex bodies such that the volume a random body from 
this family is hard to estimate from random samples.  
His result does not supersede ours in the sense that our lower bound of $2^{\Omega(\sqrt{d})}$ is stronger, and perhaps 
optimal, and our construction of the hard family is explicit.

It is known that if the convex body is a polytope with $\poly(d)$ facets, then it can be learned with 
$\poly(d)$ uniformly random samples~\cite{KV07} in an information-theoretical sense.  
However, whether this can be done efficiently remains open:
\begin{problem}
Can one learn polytopes with $\poly(d)$ facets from $\poly(d)$ uniformly random (over the polytope) samples in 
$\poly(d)$ time?
\end{problem}

\bibliography{learning}
\bibliographystyle{abbrv}    

\end{document}